\documentclass{article}

\usepackage[utf8]{inputenc} 
\usepackage[T1]{fontenc}    
\usepackage{hyperref}       
\usepackage{url}            
\usepackage{booktabs}       
\usepackage{amsfonts}       
\usepackage{nicefrac}       
\usepackage{microtype}      
\usepackage{xcolor}         

\usepackage{amsmath, amssymb, amsthm, graphicx, url} 
\usepackage[numbers]{natbib}
\usepackage[algo2e,ruled]{algorithm2e}
\usepackage{parskip} 
\usepackage{enumitem} 
\usepackage{fullpage} 
\usepackage{cleveref}       
\usepackage{thmtools, thm-restate} 

\newcommand{\cD}{\mathcal{D}}

\newcommand{\cK}{\mathcal{K}}

\newcommand{\cO}{\mathcal{O}}

\newcommand{\E}{\mathbb{E}}

\newcommand{\I}{\mathbb{I}}

\newcommand{\N}{\mathbb{N}}

\renewcommand{\P}{\mathbb{P}}

\newcommand{\R}{\mathbb{R}}

\DeclareMathOperator*{\argmax}{argmax}

\newcommand{\mutation}{value}
\newcommand{\duration}{duration}
    \newcommand{\durationrm}{\tau}
\newcommand{\decision}{decision}
\newcommand{\accept}{\mathrm{accept}}
\newcommand{\cost}{cost}
    \newcommand{\costrm}{\mathrm{cost}}
    \newcommand{\cosrm}{\mathrm{cost}}
    \newcommand{\cosbar}{\overline{c}}
\newcommand{\reward}{reward}
    \newcommand{\rewardrm}{\mathrm{reward}}
    \newcommand{\rewrm}{\mathrm{reward}}
    \newcommand{\rewbar}{\overline{r}}

\newcommand{\cape}{CAPE}
\newcommand{\esc}{ESC}
\newcommand{\capf}[1]{D_{#1}}
\newcommand{\bx}{\boldsymbol{x}}
\newcommand{\bX}{\boldsymbol{X}}
\newcommand{\e}{\varepsilon}

\newcommand{\lrb}[1]{\left(#1\right)}
\newcommand{\brb}[1]{\bigl(#1\bigr)}
\newcommand{\Brb}[1]{\Bigr(#1\Bigr)}
\newcommand{\lsb}[1]{\left[#1\right]}
\newcommand{\bsb}[1]{\bigl[#1\bigr]}
\newcommand{\Bsb}[1]{\Bigr[#1\Bigr]}
\newcommand{\lcb}[1]{\left\{#1\right\}}
\newcommand{\bcb}[1]{\bigl\{#1\bigr\}}

\newcommand{\lce}[1]{\left\lceil#1\right\rceil}
\newcommand{\bce}[1]{\bigl\lceil#1\bigr\rceil}

\newcommand{\lab}[1]{\left\lvert#1\right\rvert}
\newcommand{\bab}[1]{\bigl\lvert#1\bigr\rvert}

\newcommand{\bool}{\{0,1\}}
\newcommand{\spin}{\{-1,1\}}
\newcommand{\spino}{[-1,1]}
\newcommand{\spinon}{\spino^\N}

\newcommand{\s}{\subset}
\newcommand{\m}{\setminus}
\newcommand{\iop}{\infty}

\renewcommand{\l}{\ldots}
\newcommand{\dt}{\displaystyle}
\newcommand{\Pol}{\Pi}

\newcommand{\xbar}{\overline{x}}

\newcommand{\ks}{k^\star}

\newcommand{\kts}{k'}

\newcommand{\sbar}{\overline{c}}
\newcommand{\rp}{\widehat{r}^{\,+}}
\newcommand{\rmm}{\widehat{r}^{\,-}}
\newcommand{\rpm}{\widehat{r}^{\,\pm}}
\newcommand{\spp}{\widehat{c}^{\,+}}
\newcommand{\spm}{\widehat{c}^{\,\pm}}
\newcommand{\sm}{\widehat{c}^{\,-}}

\newcommand{\wt}{\widetilde}
\newcommand{\Nex}{N_{\mathrm{ex}}}
\newcommand{\papertitle}{ROI Maximization\\in Stochastic Online Decision-Making}
\newcommand{\suppl}{Appendix}
\newcommand{\roi}{\mathrm{ROI}}
\newcommand{\fracc}[2]{#1 / #2}

\newtheorem{theorem}{Theorem}
\newtheorem{lemma}[theorem]{Lemma}

\title{\papertitle}

\author{
  \textbf{Nicol\`o Cesa-Bianchi}\\
  \small{Universit\`a degli Studi di Milano \& DSRC}
  \and
  \textbf{Tommaso Cesari}\\
  \small{ Toulouse School of Economics (TSE) \& Artificial and Natural Intelligence Toulouse Institute (ANITI) }
  \and
  \textbf{Yishay Mansour}\\
  \small{ \hspace{5cm} Tel Aviv University \& Google research \hspace{5cm} }
  \and
  \textbf{Vianney Perchet}\\
  \small{ CREST, ENSAE \& Criteo AI Lab, Paris }
}

\begin{document}

\maketitle

\begin{abstract}%
We introduce a novel theoretical framework for Return On Investment (ROI) maximization in repeated decision-making. Our setting is motivated by the use case of companies that regularly receive proposals for technological innovations and want to quickly decide whether they are worth implementing. We design an algorithm for learning ROI-maximizing decision-making policies over a sequence of innovation proposals. Our algorithm provably converges to an optimal policy in class $\Pi$ at a rate of order $\min\big\{1/(N\Delta^2),N^{-1/3}\}$, where $N$ is the number of innovations and $\Delta$ is the suboptimality gap in $\Pi$. A significant hurdle of our formulation, which sets it aside from other online learning problems such as bandits, is that running a policy does not provide an unbiased estimate of its performance.
\end{abstract}

\section{Introduction}
\label{s:intro}
Often, companies have to make yes/no decisions, such as whether to adopt a new technology or retire an old product. However, finding out the best option in all circumstances could mean spending too much time or money in the evaluation process.
If the decisions to make are many, one could be better off making more of them quickly and inexpensively, provided that these decisions have an overall positive effect.
In this paper, we investigate the problem of determining a decision policy to balance the reward over cost ratio optimally (i.e., to maximize the return on investment).

\paragraph{A motivating example.}
Consider a technology company that keeps testing innovations to increase some chosen metric (e.g., benefits, gross revenue, revenue excluding the traffic acquisition cost). Before deploying an innovation, the company wants to figure out whether it is profitable. As long as each innovation can be tested on i.i.d.\ samples of users, the company can perform randomized tests and make statistically sound decisions. However, there is an incentive to make these tests run as quickly as possible because, for example, the testing process is expensive. Another reason could be that keeping a team on a project that has negative, neutral, or even borderline positive potential prevents it from testing other ideas that might lead to a significantly better improvement.
In other words, it is crucial to learn when to drop barely positive innovations in favor of highly positive ones, so to increase the overall flow of improvement over time (i.e., the ROI of the tests).

More generally, our framework describes problems where an agent faces a sequence of decision tasks consisting of either accepting or rejecting an innovation. Before making each decision, the agent can invest resources into reducing the uncertainty on the value brought by the innovation. 
The global objective is to maximize the total ROI. Namely, the ratio between the total value accumulated by accepting innovations and the total cost. For an in-depth discussion on alternative goals, we refer the reader to the \suppl{} (Section~\ref{s:model-choice}).

\paragraph{The model.}
Each task $n$ in the sequence is associated with a pair $(\mu_n, \mathcal{D}_n)$ that the learner can \textsl{never} directly observe.
\begin{itemize}[nosep]
\item $\mu_n$ is a random variable representing the (possibly negative) true value of the $n$-th innovation.
\item $\cD_n$ is a probability distribution over the real numbers with expectation $\mu_n$, modeling the feedback on the $n$-th innovation that the learner can gather from testing (see below).
\end{itemize}
During the $n$-th task, the learner can draw arbitrarily many i.i.d.\ samples $X_{n,1},X_{n,2},\ldots$ from $\cD_n$, accumulating information on the unknown value $\mu_n$ of the innovation currently being tested.
After stopping drawing samples, the learner can decide to either accept the innovation, earning $\mu_n$ as a reward, or reject it and gain nothing instead.
We measure the agent performance during $N$ tasks as the (expected) total amount of value accumulated by accepting innovations $\mu_n$ divided by the (expected) total number of samples requested throughout all tasks.
In Section~\ref{s:setting} we present this setting in more detail and introduce the relevant notation.

\paragraph{I.I.D.\ assumption.}
We assume that the value $\mu_n$ of the $n$-th innovation is drawn i.i.d.\ from an unknown and fixed distribution.
This assumption is meaningful if past decisions do not influence future innovations whose global quality remains stable over time.
In particular, it applies whenever innovations can progress in many orthogonal directions, each yielding a similar added value (e.g., when different teams within the same company test improvements relative to individual aspects of the company).
If both the state of the agent and that of the environment evolve, but the ratio of good versus bad innovations remains essentially the same, then this i.i.d.\ assumption is still justified.
In other words, it is not necessarily the absolute quality of the innovations that has to remain stationary, but rather the relative added value of the innovations given the current state of the system.
This case is frequent in practice, especially when a system is close to its technological limit.
Last but not least, algorithms designed under stochastic assumptions often perform surprisingly well in practice, even if i.i.d.\ assumptions are not fully satisfied or simply hard to check.

\paragraph{A baseline strategy and policy classes.}
A natural, yet suboptimal, approach for deciding if an innovation is worth accepting is to gather samples sequentially, stopping as soon as the absolute value of their running average surpasses a threshold, and then accepting the innovation if and only if the average is positive. 
The major drawback of this approach is that the value $\mu_n$ of an innovation $n$ could be arbitrarily close to zero. 
In this case, the number of samples needed to reliably determine its sign (which is of order $1/\mu_n^2$) would become prohibitively large. 
This would result is a massive time investment for an innovation whose return is negligible at best. 
In hindsight, it would have been better to reject the innovation early and move on to the next task. 
For this reason, testing processes in practice needs hard termination rules of the form: \textsl{if after drawing a certain number of samples no confident decision can be taken, then terminate the testing process rejecting the current innovation}. 
Denote by $\durationrm$ this capped early stopping rule and by $\accept$ the accept/reject decision rule that comes with it. 
We say that the pair $\pi = (\durationrm, \accept)$ is a \textsl{policy} because it fully characterize the decision-making process for an innovation.
Policies defined by capped early stopping rules (see \eqref{e:pol-hoeff1} for a concrete example) are of great practical importance \citep{johari2017peeking,kohavi2013online}.
However, policies can be defined more generally by any reasonable stopping rule and decision function.
Given a (possibly infinite) set of policies, and assuming that $\mu_1,\mu_2,\ldots$ are drawn i.i.d.\ from some unknown but fixed distribution, the goal is to learn efficiently, at the lowest cost, the best policy $\pi_{\star}$ in the set with respect to a sensible metric.
Competing against fixed policy classes is a common modeling choice that allows to express the intrinsic constraints that are imposed by the nature of the decision-making problem.
For example, even if some policies outside of the class could theoretically yield better performance, they might not be implementable because of time, budget, fairness, or technology constraints.

\paragraph{Challenges.}
One of the biggest challenges arising in our framework is that running a decision-making policy generates a collection of samples that ---in general--- cannot be used to form an unbiased estimate of the policy reward (see the impossibility result 
in Section~\ref{s:imposs} of the \suppl).
The presence of this bias is a significant departure from settings like multiarmed and firing bandits~\citep{auer2002finite,jain2018firing}, where the learner observes an unbiased sample of the target quantity at the end of every round (see the next section for additional details).
Moreover, contrary to standard online learning problems, the relevant performance measure is neither additive in the number of innovations nor in the number of samples per innovation. 
Therefore, algorithms have to be analyzed globally, and bandit-like techniques ---in which the regret is additive over rounds--- cannot be directly applied. 
We argue that these technical difficulties are a worthy price to pay in order to define a plausible setting, applicable to real-life scenarios.

\paragraph{Main contributions.}
The first contribution of this paper is providing a mathematical formalization of our ROI maximization setting for repeated decision making (\Cref{s:setting}).
We then design an algorithm called Capped Policy Elimination (Algorithm~\ref{algo:cape}, \cape) that applies to finite policy classes (\Cref{s:cape}). 
We prove that \cape{} converges to the optimal policy at rate $1/(\Delta^2 N)$, where $N$ is the number of tasks and $\Delta$ is the unknown gap between the performance of the two best policies, and at rate $N^{-1/3}$ when $\Delta$ is small (\Cref{t:early-stop}) .
In Section~\ref{s:infinitePol} we tackle the challenging problem of infinitely large policy classes.
For this setting, we design a preprocessing step (Algorithm~\ref{algo:reduction}, \esc) that leads to the \esc{}-\cape{} algorithm. 
We prove that this algorithm converges to the optimal policy in an infinite set at a rate of $N^{-1/3}$ (Theorem~\ref{t:final}).

\paragraph{Limitations.} 
Although we do not investigate lower bounds in this paper, we conjecture that our $N^{-1/3}$ convergence rate it is optimal due to similarities with bandits with weakly observable feedback graphs (see \Cref{s:cape}, ``Divided we fall'').
Another limitation of our theory is that it only applies to i.i.d.\ sequences of values $\mu_n$. It would be interesting to extend our analysis to distributions of $\mu_n$ that evolve over time.
These two intriguing problems are left open for future research.

\section{Related Work}
\label{s:related}
Return on Investment (ROI) was developed and popularized by Donaldson Brown in the early Nineties \citep{roiGuy} and it is still considered an extremely valuable metric by the overwhelming majority of marketing managers \citep{farris2010marketing}. Beyond economics, mathematics, and computer science, ROI finds applications in other fields, such as cognitive science and psychology \citep{chabris2009allocation}. Despite this, to the best of our knowledge, no theoretical online learning framework has been developed specifically for ROI maximization.
However, our novel formalization of this sequential decision problem does share some similarities with other known online learning settings.
In this section, we review the relevant literature regarding these settings and stress the differences with ours.

\paragraph{Prophet inequalities and Pandora's box.} 
In prophet inequalities \citep{lucier2017economic,correa2019recent,alaei2012online}, an agent observes sequentially (usually non-negative) random variables $Z_1, \ldots, Z_n$ and decides to stop at some time $\tau$; the reward is then $Z_\tau$. Variants include the possibility of choosing more than one random variable (in which case the reward is some function of the selected random variables), and the possibility to go back in time (to some extent). 
The Pandora's box problem is slightly different \citep{weitzman1979optimal,kleinberg2016descending,esfandiari2019online}; in its original formulation, the agent can pay a cost $c_n\geq 0$ to observe any $Z_n$. After stopping exploring, the agent's final utility is the maximum of the observed $Z_n$'s minus the cumulative cost (or, in other variants, some function of these).
Similarly to the (general) prophet inequality, the agent in our sequential problem faces random variables ($Z_n=\mu_n$ in our notation) and sequentially selects any number of them (possibly with negative values) without the possibility to go back in time and change past decisions. 
The significant difference is that the agent in our setting never observes the value of $\mu_n$. In Pandora's box, the agent can see this value by paying some price (that approximately scales as $1/\varepsilon^2$ where $\varepsilon$ is the required precision). Finally, the global reward is the cumulative sum (as in prophets) and not the maximum (as in Pandora's box) of the selected variables, normalized by the total cost (as in Pandora's box, but our normalization is multiplicative instead of additive, as it represents a ROI).

\paragraph{Multi-armed bandits.}
If we think of the set of all policies used by the agent to determine whether or not to accept innovations as arms, our setting becomes somewhat reminiscent of multi-armed bandits \citep{Slivkins19book,bubeck2012regret,rosenberg2007social}.
However, there are several notable differences between these two problems.
In stochastic bandits, the agent observes an unbiased estimate of the expected reward of each pulled arm.
In our setting, the agent not only does not see it directly, but it is mathematically impossible to define such an estimator solely with the feedback received (see the impossibility result in Section~\ref{s:imposs} of the \suppl{}).
Hence, off-the-shelf bandit algorithms cannot be run to solve our problem.
In addition, the objective in bandits is to maximize the cumulative reward, which is additive over time, while the ROI is not.
Thus, it is unclear how formal guarantees for bandit algorithms would translate to our problem.

We could also see firing bandits \citep{jain2018firing} as a variant of our problem, where $\mu_n$ belongs to $[0,1]$,  $\mathcal{D}_n$ are Bernoulli distribution with parameter $\mu_n$, and policies
have a specific form that allows to easily define unbiased estimates of their rewards (which, we reiterate, is not possible in our setting in general).
Furthermore, in firing bandits, it is possible to go back and forth in time, sampling from any of the past distributions $\mathcal{D}_n$ and gathering any number of samples from it.
This is a reasonable assumption for the original motivations of firing bandits because the authors thought of $\mu_n$ as the value of a project in a crowdfunding platform, and, in their setting, drawing samples from $\mathcal{D}_n$ corresponds to displaying projects on web pages.
However, in our setting, $\mu_n$ represents the theoretical increment (or decrement) of a company's profit through a given innovation, and it is unlikely that a company would show new interest in investing in a technology that has been tested before and did not prove to be useful (a killed project is seldom re-launched).
Hence, when the sampling of  $\mathcal{D}_n$ stops, an irrevocable decision is made. After that, the learner cannot draw any more samples in the future.
Finally, as in multi-armed bandits, the performance criterion in firing bandits is the cumulative reward and not the global ROI.

Another online problem that shares some similarities with ours is bandits with knapsacks \citep{badanidiyuru2018bandits}. In this problem, playing an arm consumes one unit of time together with some other resources, and the learner receives an unbiased estimate of its reward as feedback. The process ends as soon as time or any one of the other resources is exhausted. As usual, the goal is to maximize the cumulative regret. As it turns out, we can also think of our problem as a budgeted problem. In this restatement, there is a budget of $T$ samples. The repeated decision-making process proceeds as before, but it stops as soon as the learner has drawn a total of $T$ samples across all decision tasks. The goal is again to maximize the total expected reward of accepted innovations divided by $T$ (see Section~\ref{s:model-choice} of the \suppl{} for more details on the reduction). As per the other bandit problems, there are two crucial differences. First, running a policy does not reveal an unbiased estimate of its reward. Second, our objective is different, and regret bounds do not directly imply convergence to optimal ROI.

\paragraph{Repeated A/B testing.}
We can view our problem as a framework for repeated A/B testing \citep{tukey1953problem,genovese2006false,foster2008alpha,heesen2016dynamic,javanmard2018online,azevedo2018b,li2019multiple,schmit2019optimal}, in which assessing the value of an innovation corresponds to performing an A/B test, and the goal is maximizing the ROI. 
A popular metric to optimize sequential A/B tests is the so-called \textsl{false discovery rate} (FDR) ---see \citep{ramdas2017online,yang2017framework} and references therein. Roughly speaking, the FDR is the ratio of accepted $\mu_n$ that are negative over the total number of accepted $\mu_n$ (or more generally, the number of incorrectly accepted tests over the total number if the metric used at each test changes with time). 
This, unfortunately, disregards the relative values of tests $\mu_n$ that must be taken into account when optimizing a single metric \citep{chen2019contextual,robertson2018online}. Indeed, the effect of many even slightly negative accepted tests could be overcome by a few largely positive ones. 
For instance, assume that the samples $X_{n,i}$ of any distribution $\mathcal{D}_n$ belong to $\spin$, and that their expected value $\mu_n$ is uniformly distributed on $\{-\e,\e\}$.
To control the FDR, each A/B test should be run for approximately $1/\e^2$ times, yielding a ratio of the average value of an accepted test to the number of samples of order  $\varepsilon^3$. A better strategy, using just one sample from each A/B test, is simply to accept $\mu_n$ if and only if the first sample is positive. Direct computations show that this policy, which fits our setting, achieves a significantly better performance of order $\e$.

Some other A/B testing settings are more closely related to ours, but make stronger additional assumptions or suppose preliminary knowledge: for example,  smoothness assumptions can be made on both $\mathcal{D}_n$ and the distributions of $\mu_n$ \citep{azevedo2018b}, or the distribution of $\mu_n$ is known, and the distribution of samples belongs to a single parameter exponential family, also known beforehand \citep{schmit2019optimal}.

\paragraph{Rational metareasoning.}
Our setting is loosely related to the AI field of meta-reasoning \citep{griffiths2019doing,meta1}. In a metalevel decision problem, determining the utility (or reward) of a given action is computationally intractable. Instead, the learner can run a simulation, investing a computational cost to gather information about this hidden value. The high-level idea is then to learn \textsl{which} actions to simulate. After running some simulations, the learner picks an action to play, gains the corresponding (hidden) reward, and the state of the system changes. In rational meta-reasoning, the performance measure is the value of computation (VOC): the difference between the increment in expected utility gained by executing a simulation and the cost incurred by doing so. This setting is not directly comparable to ours for two reasons. First, the performance measure is different, and the additive nature of the difference that defines the VOC gives no guarantees on our multiplicative notion of ROI. Second, in this problem, one can pick which actions to simulate, while in our settings, innovations come independently of the learner, who has to evaluate them in that order.

\section{Setting and Notation}
\label{s:setting}
In this section, we formally introduce the repeated decision-making protocol for an agent whose goal is to maximize the total return on investment in a sequence of decision tasks.

The only two choices that an agent makes in a decision task are when to stop gathering information on the current innovation and whether or not to accept the innovation based on this information. 
In other words, the behavior of the agent during each task is fully characterized by the choice of a pair
$
    \pi 
= 
    (\durationrm, \accept)
$
that we call a \textsl{(decision-making) policy} (for the interested reader, Section~\ref{s:policyFormal} of the \suppl{} contains a short mathematical discussion on policies), where:
\begin{itemize}
    \item $\durationrm(\bx)$, called \textsl{\duration{}}, maps a sequence of observations $\bx = (x_1,x_2,\ldots)$ to an integer $d$ (the no.\ of observations after which the learner stops gathering info on the current innovation);
    \item $\accept(d,\bx)$, called \textsl{\decision{}}, maps the firs $d$ observations of a sequence $\bx = (x_1,x_2,\ldots)$ to a boolean value in $\{0,1\}$ (where $1$ represents accepting the current innovation).
\end{itemize}

An instance of our repeated decision-making problem is therefore determined by a set of admissible policies $\Pol = \{\pi_k\}_{k\in\cK}$ $=  \lcb{ (\durationrm_k, \, \accept) }_{k\in\cK}$ (with $\cK$ finite or countable) and a distribution $\mu$ on
$[-1,1]$, modelling the value of innovations.%
\footnote{%
We assume that the values of the innovations and the learner's observations belong to $[-1,1]$ and $\{-1,1\}$ respectively. 
We do this merely for the sake of readability (to avoid carrying over awkward constants or distributions $\cD_n$).
With a standard argument, both $[-1,1]$ and $\{-1,1\}$ can be extended to arbitrary codomains straightforwardly under a mild assumption of subgaussianity.
}
Naturally, the former is known beforehand but the latter is unknown and should be learned.

For a fixed choice of $\Pol$ and $\mu$, the protocol is formally described below.
In each decision task $n$:
\begin{enumerate}[topsep = 0pt, parsep = 0pt, itemsep = 0pt]
    \item \label{i:condIndSamples-1} the \textsl{\mutation{}} $\mu_n$ of the current innovation is drawn i.i.d.\ according to $\mu$; 
    \item \label{i:condIndSamples-2} $\bX_n$ is a sequence of i.i.d.\ (given $\mu_n$) \textsl{observations} with $X_{n,i}=\pm 1$ and $\E[X_{n,i}\mid\mu_n]=\mu_n$;
    \item \label{i:duratio} the agent picks   $k_n\in\cK$ or, equivalently, a policy $\pi_{k_n} = (\durationrm_{k_n}, \accept ) \in \Pol$;
    \item \label{i:samples} the agent draws the first $d_n = \durationrm_{k_n}(\bX_n)$ \textsl{samples}\footnote{Given $\mu_n$, the random variable $d_n$ is a stopping time w.r.t.\ the natural filtration associated to $\bX_n$.} of the sequence of observations $\bX_n$;
    \item \label{i:accept} on the basis of these sequential observations, the agent makes the decision $\accept \brb{ d_n, \bX_n  }$.
\end{enumerate}
Crucially, $\mu_n$ is \textsl{never} revealed to the learner.
We say that the agent \textsl{runs a policy} $\pi_{k} = (\durationrm_{k},\accept)$ (on a \mutation{} $\mu_n$) when steps \ref{i:samples}--\ref{i:accept} occur (with $k_n \gets k$). We also say that they accept (resp., rejects) $\mu_n$ if their decision at step \ref{i:accept} is equal to $1$ (resp., $0$). 
Moreover, we say that the \textsl{\reward} obtained and the  \textsl{\cost{}} payed by running a policy $\pi_k = (\durationrm_k, \accept)$ on a \mutation{} $\mu_n$ are, respectively,
\begin{equation}
    \label{e:reward}
    \rewardrm(\pi_k, \mu_n)
 =
    \mu_n \,  \accept \brb{ \durationrm_k(\bX_n), \bX_n  }
    \in \{\mu_n,0\}
\hspace{6ex}
    \costrm(\pi_k, \mu_n)
 =
    \durationrm_k(\bX_n)
    \in \N
\end{equation}
The objective of the agent is to converge to the highest ROI of a policy in $\Pol$, i.e., to guarantee that
\begin{equation}
\label{e:regret}
    R_N
=
    \sup_{k \in \cK} \frac{\sum_{n=1}^N \E \bsb{ \rewardrm(\pi_{k},\mu_n) } }
        {\sum_{m=1}^N \E \bsb{ \costrm(\pi_{k},\mu_m) } }
    - \frac{\sum_{n=1}^N \E \bsb{ \rewardrm(\pi_{k_n},\mu_n) } }
        {\sum_{m=1}^N \E \bsb{ \costrm(\pi_{k_m},\mu_m) } }
    \to 0 \quad \text{as } N\to \iop
\end{equation}
where the expectations are taken with respect to $\mu_n$, $\bX_n$, and (possibly) the random choices of $k_n$.

To further lighten notations, we denote the expected \reward{}, \cost{}, and ROI of a policy $\pi$ by
\begin{equation}
    \label{e:rewCostShort}
    \rewrm(\pi) = \E \bsb{ \rewardrm(\pi, \mu_n) }
,\,
    \cosrm(\pi) = \E \bsb{ \costrm(\pi, \mu_n) }
,\,
    \roi(\pi) = { \rewrm(\pi) }/{ \cosrm(\pi) }
\end{equation}
respectively and we say that $\pi_{k^\star}$ is an \textsl{optimal policy} if 
$
	k^\star \in \argmax_{k\in \cK} \roi(\pi_k)
$.
Note that $\rewrm(\pi)$ and $\cosrm(\pi)$ do not depend on $n$ because $\mu_n$ is drawn i.i.d. according to $\mu$.

For each policy $(\durationrm,\accept)\in \Pol$ and all tasks $n$, we allow the agent to reject the \mutation{} $\mu_n$ regardless of the outcome of the sampling. Formally, the agent can always run the policy $(\durationrm, 0)$, where the second component of the pair is the decision identically equal to zero (i.e., the rule ``always reject''). 

We also allow the agent to draw arbitrarily many extra samples in addition to the number $\durationrm(\bX_n)$ that they would otherwise draw when running a policy $(\durationrm,\accept)\in \Pol$ on a value $\mu_n$, provided that these additional samples are not taken into account in the decision to either accept or reject $\mu_n$. 
Formally, the agent can always draw $\durationrm(\bX_n) + k$ many samples (for any $k\in \N$) before making the decision $\accept \brb{ \durationrm(\bX_n), \bX_n  }$, where we stress that the first argument of the decision function $\accept$ is $\durationrm(\bX_n)$ and not $\durationrm(\bX_n)+k$.
Oversampling this way worsens the objective and might seem utterly counterproductive, but it will be crucial for recovering unbiased estimates of $\mu_n$.

\section{Competing Against \texorpdfstring{$K$}{K} policies (\cape)} 
\label{s:cape}
As we mentioned in the introduction, in practice the \duration{} of a decision task is defined by a capped early-stopping rule ---e.g., drawing samples until $0$ falls outside of a confidence interval around the empirical average, or a maximum number of draws has been reached.
More precisely, if $N$ tasks have to be performed, one could consider the natural policy class $\lcb{ (\durationrm_k, \accept) }_{k \in \{1,\l,K\}}$ given by 
\begin{equation}
    \label{e:pol-hoeff1}
    \durationrm_k(\bx)
 =
    \min \brb{k,\ \inf \lcb{ d \in \N : \lab{ \xbar_d } \ge \alpha_d  } }
\qquad \text{and} \qquad 
    \accept(d, \bx)
 =
    \I \lcb{ \xbar_d \ge \alpha_d  }
\end{equation}
where $\xbar_d = (\nicefrac{1}{d})\sum_{i= 1}^d x_i$ is the average of the first $d$ elements of the sequence $\bx= (x_1, x_2, \l)$ and $\alpha_d = c \sqrt{(\nicefrac{1}{d})\ln(KN/\delta)}$, for some $c>0$ and $\delta\in(0,1)$.
While in this example policies are based on an Hoeffding concentration rule, in principle the learner is free to follow any scheme.
Thus, we now generalize this notion and present an algorithm with provable guarantees against these finite families of policies.

\paragraph{Finite sets of policies.}
In this section, we focus on finite sets of $K$ policies 
$
	\Pol
 =
	\{ \pi_k \}_{k\in\{1,\l,K\}}
 =
    \lcb{ (\durationrm_k, \accept) }_{k\in\{1,\l,K\}}
$
where $\accept$ is an arbitrary \decision{} and $\durationrm_1,\l,\durationrm_K$ is any sequence of bounded \duration{}s (say, $\durationrm_k \le k$ for all $k$).\footnote{We chose $\durationrm_k \le k$ for the sake of concreteness. All our results can be straightforwardly extended to arbitrary $\durationrm_k \le D_k$ by simply assuming without loss of generality that $k\mapsto D_k$ is monotone and replacing $k$ with $D_k$.}
For the sake of convenience, we assume the \duration{}s are sorted by index ($\durationrm_k \le \durationrm_h$ if $k\le h$), so that $\durationrm_1$ is the shortest and $\durationrm_K$ is the longest. 

\paragraph{Divided we fall.}
A common strategy in online learning problems with limited feedback is explore-then-commit (ETC).
ETC consists of two phases. In the first phase (explore), each action is played for the same amount of rounds, collecting this way i.i.d.\ samples of all rewards. In the subsequent commit phase, the arm with the best empirical observations is played consistently.
Being very easy to execute, this strategy is popular in practice, but unfortunately, it is theoretically suboptimal in some applications. 
A better approach is performing action elimination. In a typical implementation of this strategy, all actions in a set are played with a round-robin schedule, collecting i.i.d.\ samples of their rewards. At the end of each cycle, all actions that are deemed suboptimal are removed from the set, and a new cycle begins.
Neither one of these strategies can be applied directly because running a policy in our setting does not return an unbiased estimate of its reward (for a quick proof of this simple result, see Section~\ref{s:imposs} in the \suppl{}). However, it turns out that we can get an i.i.d. estimate of a policy $\pi$ by playing a \textsl{different} policy $\pi'$. Namely, one that draws \textsl{more} samples than $\pi$. This is reminiscent of bandits with a weakly observable feedback graph, a related problem for which the time-averaged regret over $T$ rounds vanishes at a $T^{-1/3}$ rate \citep{alon2015online}.
Albeit none of these three techniques works on its own, suitably interweaving all of them does.

\paragraph{United we stand.}
With this in mind, we now present our simple and efficient algorithm (Algorithm~\ref{algo:cape}, \cape) whose ROI converges (with high probability) to the best one in a finite family of policies.
We will later discuss how to extend the analysis even further, including countable families of policies.
Our algorithm performs policy elimination (lines~\ref{a:pe-begin}--\ref{s:polel}) for a certain number of tasks (line~\ref{a:pe-begin}) or until a single policy is left (line~\ref{s:testPolElim}). 
After that, it runs the best policy left in the set (line~\ref{s:exploit}) for all remaining tasks.
During each policy elimination step, the algorithm oversamples (line~\ref{a:oversampling}) by drawing twice as many samples as it would suffice to take its decision $\accept \brb{ \durationrm_{{\max(C_n)}}(\bX_n), \bX_n }$ (at line~\ref{s:decision}).
These extra samples are used to compute rough estimates of \reward{}s and \cost{}s of all potentially optimal policies and more specifically to build \textsl{unbiased} estimates of these \reward{}s.
The test at line~\ref{s:nexbig} has the only purpose of ensuring that the denominators $\sm_n(k)$ at line~\ref{s:polel} are bounded away from zero so that all quantities are well-defined.

{ 
\renewcommand{\capf}{}

\begin{algorithm2e}
    \LinesNumbered
    \SetAlgoNoLine
    \SetAlgoNoEnd
    \DontPrintSemicolon
	\SetKwInput{kwInit}{Initialization}
	\KwIn{finite policy set $\Pol$,
	number of tasks $N$, confidence parameter $\delta$, exploration cap $\Nex$}
	\kwInit{let $C_1 \gets \{1,\l,K\}$ be the set of indices of all currently optimal candidates}
\For 
    {%
    task $n = 1,\ldots,\Nex$%
    \nllabel{a:pe-begin}%
    }
    {
    draw the first $2  \capf{ \max(C_n) }$ samples $X_{n,1},\l,X_{n,2\capf{\max(C_n)}}$ of $\bX_n$ \nllabel{a:oversampling}\;
    make the \decision{} $\accept \brb{ \durationrm_{{\max(C_n)}}(\bX_n), \bX_n }$\nllabel{s:decision}\;
    \lIf
        {%
        \nllabel{s:nexbig}%
        $n \ge 2 \capf{K}^2 \ln(4K\Nex/ \delta)$%
        }
        {%
        let 
        $C_{n+1} \gets C_n \m C_n'$, where \vspace{1ex}
        
        $
            \quad
            C_n'
	    =
	        \left\{ k \in C_n \,:\,
            \brb{ 
            \rp_n(k) \ge 0 \text{ and }  \fracc{\rp_n(k)}{\sm_n(k)} < \fracc{\rmm_n(j)}{\spp_n(j)} 
            \text{, for some } j\in C_n
            }
            \right.
        $
        $
            \left. \hspace{22.58mm}
            \text{ or }
            \brb{
            \rp_n(k) < 0 \text{ and } 
            \fracc{\rp_n(k)}{\spp_n(k)} < \fracc{\rmm_n(j)}{\sm_n(j)} 
            \text{, for some } j\in C_n
            }
            \right\}
        $
        \vspace{-0.5ex}
	    \begin{align}
	        \label{e:rpmdef}
            \rpm_n(k)
        & =
    	    \frac{1}{n} \sum_{m=1}^{n}
	        \sum_{i=1}^{\capf{\max(C_m)}} \frac{X_{m,\capf{\max(C_m)}+i}}{\capf{\max(C_m)}}
	        \, 
    	    \accept \brb{ \durationrm_k (\bX_m), \bX_m}
    	    \pm \sqrt{\frac{2}{n}\ln \frac{4 K \Nex}{\delta}}
        \\
            \label{e:cpmdef}
        	\spm_n(k)
        & =
    	    \frac{1}{n} \sum_{m=1}^{n}
    	    \durationrm_k(\bX_m) \pm (\capf{k}-1) \sqrt{\frac{1}{2n}\ln\frac{4 K \Nex }{ \delta }}
        \end{align}
        \nllabel{s:polel}\vspace{-3ex}
        }
\lIf
    {
    \nllabel{s:testPolElim}
    $\lab{C_{n+1}} = 1$
    }
    {
    let $\rpm_{\Nex}(k)\gets\rpm_n(k)$, $\spm_{\Nex}(k)\gets\spm_n(k)$, $C_{\Nex+1} \gets C_{n+1}$,  \textbf{break}
    }
	}
run policy $\pi_{\kts}$ for all remaining tasks, where 
\begin{equation}
    \label{e:ktsdef}
    \kts \in 
\begin{cases}
    \dt{ \argmax_{k \in C_{\Nex+1}}  \lrb{ {\rp_{\Nex}(k)} / {\sm_{\Nex}(k)} } }
        & \text{ if } \rp_{\Nex}(k)\ge 0 \text{ for some } k \in C_{\Nex+1}
    \\
    \dt{ \argmax_{k \in C_{\Nex+1}}  \lrb{ {\rp_{\Nex}(k)} / {\spp_{\Nex}(k)} } }
        & \text{ if } \rp_{\Nex}(k) < 0 \text{ for all } k \in C_{\Nex+1}
\end{cases}
\end{equation}
\nllabel{s:exploit}
\caption{Capped Policy Elimination (CAPE)\label{algo:cape}}
\end{algorithm2e}

As usual in online learning, the \textsl{gap} in performance between optimal and sub-optimal policies is a complexity parameter. We define it as
$
	\Delta 
= 
	\min_{k \neq \ks}
	\brb{ \roi(\pi_{\ks}) - \roi(\pi_k) }
$,
where we recall that $\ks \in \argmax_k \roi(\pi_k)$ is the index of an optimal policy.
Conventionally, we set $1/0 = \iop$.
\begin{theorem}
\label{t:early-stop}
If $\Pol$ is a finite set of $K$ policies, then 
the ROI of Algorithm~\ref{algo:cape} run for $N$ tasks with exploration cap $\Nex = \bce{ N^{2/3} }$ and confidence parameter $\delta\in(0,1)$ converges to the optimal $\roi(\pi_{\ks})$, with probability at least $1-\delta$, at a rate
\begin{equation*}
    R_N
=
    \wt{\cO} \lrb
    {
    \min \lrb
    { 
    \frac{ K^3 }{\Delta^2 \,  N}
    , 
    \frac{ K }{ N^{1/3} } 
    } 
    }
\end{equation*}
as soon as  $N \ge K^3$ (where the $\wt{\cO}$ notation hides only logarithmic terms, including a $\log(1/\delta)$ term). 
\end{theorem}

\begin{proof}[Proof sketch.]
This theorem relies on four technical lemmas (Lemmas~\ref{lm:claim1}-\ref{lm:claim4}) whose proofs are deferred to Section~\ref{s:techlem} of the \suppl{}.

With a concentration argument (Lemma~\ref{lm:claim1}), we leverage the definitions of $\rpm_n(k), \spm_n(k)$ 
and the i.i.d.\ assumptions on the samples $X_{n,i}$ 
to show  that, with probability at least $1-\delta$, the event
\begin{equation}
    \label{e:claim1-0}
    \rmm_n(k) \le \rewrm(\pi_k) \le \rp_n(k)
\qquad\text{and}\qquad
    \sm_n(k) \le \cosrm(\pi_k) \le \spp_n(k)
\end{equation}
occurs simultaneously for all $n \le \Nex$ and all $k\le\max(C_n)$. 
For the rewards, the key is oversampling, because $\accept \brb{ \durationrm_k (\bX_m), \bX_m}$ in \cref{e:rpmdef} depends only on the first $k\le \max(C_m)$ samples of $\bX_m$ and is therefore independent of $X_{m,\capf{\max(C_m)}+i}$ for all $i$.
Assume now that \eqref{e:claim1-0} holds.

If $\Delta>0$ (i.e., if there is a unique optimal policy), we then obtain (Lemma~\ref{lm:claim2}) that suboptimal policies are eliminated after at most $\Nex'$ tasks, where
$
    \Nex'
\le 
    {288 \, \capf{K}^2 \ln(4K\Nex /\delta)}/{\Delta^2} + 1
$.
To prove it we show that a confidence interval for $\roi(\pi_k) = {\rewrm(\pi_k)}/{\cosrm(\pi_k)}$ is given by
\[
    \lsb{ 
    \frac{\rmm_n(k)}{\spp_n(k)}\I\bcb{\rp_n(k) \ge 0}
    + \frac{\rmm_n(k)}{\sm_n(k)}\I\bcb{\rp_n(k) < 0}
, \ 
    \frac{\rp_n(k)}{\sm_n(k)}\I\bcb{\rp_n(k) \ge 0} 
    + \frac{\rp_n(k)}{\spp_n(k)}\I\bcb{\rp_n(k) < 0}
    }
\]
we upper bound its length,
and  we compute an $\Nex'$ such that this upper bound is smaller than $\Delta/2$.

Afterwards, we analyze separately the case in which the test at line~\ref{s:testPolElim} is true for some task $\Nex' \le \Nex$ and its complement (i.e., when the test is always false).

In the first case, by \eqref{e:claim1-0} there exists a unique optimal policy, i.e., we have that $\Delta>0$. 
This is where the policy-elimination analysis comes into play.
We can apply the bound above on $\Nex'$, obtaining a deterministic upper bound $\Nex''$ on the number $\Nex'$ of tasks needed to identify the optimal policy. Using this upper bound, writing the definition of $R_N$, and further upper bounding (Lemma~\ref{lm:claim3}) yields
\begin{equation}
    \label{e:claim3-0}
    R_N
\le
    \min \lrb
    {\frac{(2 \capf{K}  + 1) \Nex }{N}
    ,\
    \frac
    {( 2\capf{K}  + 1 ) \brb{ 288 \, ( {\capf{K}}/{\Delta} )^2 \ln(4K\Nex /\delta) + 1 }}
    {N}
    }
\end{equation}
Finally, we consider the case in which the test at line~\ref{s:testPolElim} is false for all tasks $n \le \Nex$, and line~\ref{s:exploit} is executed with $C_{\Nex+1}$ containing two or more policies.
This is covered by a worst case explore-then-commit analysis.
The key idea here is to use the definition of $\kts$ in Equation \eqref{e:ktsdef} to lower-bound $\rewrm(\pi_{k'})$ in terms of $\rewrm(\pi_{\ks})/\cosrm(\pi_{\ks})$. 
This, together with some additional technical estimations (Lemma~\ref{lm:claim4}) leads to the result.
\end{proof}

\section{Competing Against Infinitely Many Policies (\esc-\cape)}
\label{s:infinitePol}
\Cref{t:early-stop} provides theoretical guarantees on the convergence rate $R_N$ of \cape{} to the best ROI of a finite set of policies. 
Unfortunately, the bound becomes vacuous when the cardinality $K$ of the policy set is large compared to the number of tasks $N$.
It is therefore natural to investigate whether the problem becomes impossible in this scenario.

\paragraph{Infinite sets of policies.}
With this goal in mind, we now focus on policy sets
$
	\Pol
 =
	\{ \pi_k \}_{k\in\cK}
 =
    \bcb{ (\durationrm_k, \accept) }_{k\in\cK}
$
as in the previous section, with $\cK = \N$ rather than $\cK = \{1,\ldots,K\}$.

We will show how such a countable set of policies can be reduced to a finite one containing all optimal policies with high probability (Algorithm~\ref{algo:reduction}, \esc).
After this is done, we can run \cape{} on the smaller policy set, obtaining theoretical guarantees for the resulting algorithm.

\paragraph{Estimating rewards and costs.}
Similarly to \cref{e:rpmdef,,e:cpmdef}, we first introduce estimators for our target quantities.
If at least $2\capf{k}$ samples are drawn during each of $n_2$ consecutive tasks $n_1 +1,$ $\ldots,$ $n_1+n_2$, we can define, for all $\e>0$, 
the following lower confidence bound on $\rewrm(\pi_k)$:
\begin{align}
        \label{e:ulcbdef}
    	\rmm_{k}(n_1,n_2,\e)
& =
	\frac{1}{n_2} \sum_{n=n_1+1}^{n_1+n_2}
	\sum_{i=1}^{\capf{k}}
	\frac{X_{n,\capf{k} + i}}{\capf{k}}
	\, 
	\accept \brb{ \durationrm_k (\bX_n), \bX_n}
	- 2 \e
\end{align}
If at least $\durationrm_k(\bX_n)$ samples are drawn during each of $m_0$ consecutive tasks $n_0+1, \l, n_0 + m_0$, we can define the following empirical average of $\costrm(\pi_k)$:
\begin{equation}
    \label{e:cbardef}
    	\sbar_k(n_0,m_0)
=
	\brb{ \durationrm_k(\bX_{n_0+1}) + \l + \durationrm_k(\bX_{n_0+m_0}) } / m_0
\end{equation}

\paragraph{A key observation.}
The key idea behind Algorithm~\ref{algo:reduction} (\esc) is simple. 
Since all optimal policies $\pi_{\ks}$ have to satisfy the relationships $\rewrm(\pi_k)/\cosrm(\pi_k) \le \rewrm(\pi_{\ks})/\cosrm(\pi_{\ks}) \le 1/\cosrm(\pi_{\ks})$, then, for all policies $\pi_k$ with positive $\rewrm(\pi_k)$, the \cost{} of any optimal policy $\pi_{\ks}$ must satisfy the relationship $\cosrm(\pi_{\ks}) \le \cosrm(\pi_k)/\rewrm(\pi_k)$.
In other words, \textsl{optimal policies cannot draw too many samples} and their \cost{} can be controlled by estimating the \reward{} and \cost{} of \textsl{any} policy with positive \reward{}.

We recall that running a policy $(\tau,0)$ during a task $n$ means drawing the first $\tau(\bX_n)$ samples of $\bX_n=(X_{n,1},X_{n,2},\ldots)$ and always rejecting $\mu_n$, regardless of the observations.

\begin{algorithm2e}
    \LinesNumbered
    \SetAlgoNoLine
    \SetAlgoNoEnd
    \DontPrintSemicolon
	\SetKwInput{kwInit}{Initialization}
	\KwIn{countable policy set $\Pol$, number of tasks $N$, confidence parameter $\delta$, accuracy levels $(\e_n)_n$
	
	}
	\kwInit{for all $j$, let $m_j \gets \lce{ {\ln \brb{ j(j+1)/\delta }} / {2\e_j^2} }$ and $M_j = m_1 + \l + m_j$}
\For
    {
    $j=1,2,\l$
    \nllabel{a:1-begin}
    }
    {
	run policy $\brb{ 2 \cdot \capf{2^j} , 0 }$ for $m_j$ tasks and compute $\rmm_{2^j} \gets \rmm_{2^j}(M_{j-1},m_j, \e_j)$\nllabel{s:rj}  as in  \eqref{e:ulcbdef}\;
	\textbf{if} \ \nllabel{st:test-p1} $\rmm_{2^j} > 0$ \ \textbf{then} \ let $j_0 \gets j$ and $k_0 \gets 2^{j_0}$\;\nllabel{st:found-pos}
    \Indp \For{\nllabel{st:for-loop}$l=j_0+1, j_0 +2 ,\l$}
		{
			run policy $\brb{ \durationrm_{2^l}, 0 }$ for $m_l$ tasks and compute $\sbar_{2^l} \gets \sbar_{2^l}(M_{l-1},m_l)$ as in \eqref{e:cbardef} \nllabel{s:samples2}\;
			\lIf{$\sbar_{2^l} > \capf{2^l} \,  \e_l + \capf{k_0} / \rmm_{k_0}$\nllabel{st:test-p2} }
			{
			let $j_1 \gets l$ and \textbf{return} $K \gets 2^{j_1}$ \nllabel{a:nontrivial-bound}
			\nllabel{a:1-end}
			}
		}
    } 
\caption{\label{algo:reduction} Extension to Countable (\esc)}
\end{algorithm2e}

Thus, Algorithm~\ref{algo:reduction} (\esc) first finds a policy $\pi_{k_0}$ with $\rewrm(\pi_{k_0})>0$ (lines~\ref{a:1-begin}--\ref{st:found-pos}), memorizing an upper estimate $\capf{k_0} / \rmm_{k_0}$ of the ratio $\cosrm(\pi_{k_0})/\rewrm(\pi_{k_0}) = 1/\roi(\pi_{k_0})$. 
By the argument above, this estimate upper bounds the expected number of samples $\cosrm(\pi_{\ks})$ drawn by \textsl{all} optimal policies $\pi_{\ks}$. 
Then \esc{} simply proceeds to finding the smallest (up to a factor of $2$) $K$ such that $\cosrm(\pi_{K}) \ge \capf{k_0} / \rmm_{k_0}$ (lines~\ref{st:for-loop}--\ref{a:nontrivial-bound}). Being $\capf{k_0} / \rmm_{k_0} \ge \cosrm(\pi_{k_0})/\rewrm(\pi_{k_0}) \ge \cosrm(\pi_{\ks})$ by construction, the index $K$ determined this way upper bounds $\ks$ for all optimal policies $\pi_{\ks}$. (All the previous statements are intended to hold with high probability.)
This is formalized in the following key lemma, whose full proof we defer to Section~\ref{s:techlem-countable} of the \suppl.
\begin{restatable}{lemma}{lemmaKcountable}
\label{p:k2bound-main}
Let $\Pol$ be a countable set of policies. If \esc{} is run with $\delta\in(0,1)$, $\e_1,\e_2,\l\in (0,1]$, and halts returning $K$, then
$
    \ks \le K
$
for all optimal policies $\pi_{\ks}$ with probability at least $1-\delta$.
\end{restatable}
Before proceeding with the main result of this section, we need a final lemma upper bounding the expected cost of our \esc{} algorithm.
This step is crucial to control the total ROI because in this setting with arbitrarily long \duration{}s, picking the wrong policy even once is, in general, enough to drop the performance of an algorithm down to essentially zero, compromising the convergence to an optimal policy.
This is another striking difference with other common online learning settings like stochastic bandits, where a single round has a negligible influence on the overall performance of an algorithm.
To circumvent this issue, we designed \esc{} so that it tests shorter \duration{}s first, stopping as soon as the previous lemma applies, and a finite upper bound $K$ on $\ks$ is determined.

\begin{lemma}
\label{lm:sampleCompl}
Let $\Pol$ be a countable set of policies. If \esc{} is run with $\delta\in(0,1)$, $\e_1,\e_2,\l \in (0,1]$, and halts returning $K$, 
then the total number of samples it draws before stopping (i.e., its cost) is upper bounded by
$
    \wt{\cO} \brb{ ( \nicefrac{\capf{K}}{\e^2} ) \log (\nicefrac{1}{\delta}) }
$,
where $\e = \min\{\e_1, \e_2,\l,\e_{\log_2 K}\}$.
\end{lemma}
\begin{proof}
Note that, by definition, $\e = \min\{\e_1,\e_2,\l,\e_{j_1}\} > 0$.
Algorithm~\ref{algo:reduction} (\esc) draw samples only when lines~\ref{s:rj}~or~\ref{s:samples2} are executed. 
Whenever line~\ref{s:rj} is executed ($j=1,\l,j_0$) the algorithm performs $m_j$ tasks drawing $2 \cdot \capf{2^j}$ samples each time.
Similarly, whenever line~\ref{s:samples2} is executed ($l=j_0+1,\l,j_1$) the algorithm draws at most $\capf{2^l}$ samples during each of the $m_l$ tasks.
Therefore, recalling that $j_1 = \log_2 K$, the total number of samples drawn by \esc{} before stopping is at most
\[
	\sum_{j=1}^{j_0} 2 \cdot \capf{2^j} m_j + \sum_{l=j_0+1}^{j_1} \capf{2^l} m_l
\le
	2\sum_{j=1}^{j_1} \capf{2^j} m_j 
\le
    2 j_1 \capf{2^{j_1}} \lce{ \frac{1}{2\e^2} \ln \frac{j_1(j_1+1)}{\delta} }
    \qedhere
\] 
\end{proof}
\paragraph{The \esc{}-\cape{} algorithm.}
We can now join together our two algorithms obtaining a new one, that we call \esc-\cape{}, which takes as input a countable policy set $\Pol$, the number of tasks $N$, a confidence parameter $\delta$, some accuracy levels $\e_1,\e_2,\l$, and an exploration cap $\Nex$. The joint algorithm runs \esc{} first with parameters $\Pol,N,\delta,\e_1,\e_2,\l$. Then, if \esc{} halts returning $K$, it runs \cape{} with parameters $\lcb{ (\durationrm_k, \accept )}_{k\in \{1,\ldots,K\}},N,\delta,\Nex$.

\paragraph{Analysis of \esc{}-\cape{}.}
Since \esc{} rejects all \mutation{}s $\mu_n$, the sum of the rewards accumulated during its run is zero. 
Thus, the only effect that \esc{} has on the convergence rate $R_N$ of \esc{}-\cape{} is an increment on the total cost in the denominator of its ROI.
We control this cost by minimizing its upper bound in Lemma~\ref{lm:sampleCompl}. This is not a simple matter of taking all $\e_j$'s as large as possible. 
Indeed, if all the $\e_j$'s are large, the \textbf{if} clause at line~\ref{st:test-p1} might never be verified.
In other words, the returned index $K$ depends on $\e$ and grows unbounded in general as $\e$ approaches $1/2$. This follows directly from the definition of our lower estimate on the rewards \eqref{e:ulcbdef}.
Thus, there is a trade-off between having a small $K$ (which requires small $\e_j$'s) and a small $1/\e^2$ to control the cost of \esc{} (for which we need large $\e_j$'s).
A direct computation shows that picking constant accuracy levels $\e_j = N^{-1/3}$ for all $j$ achieves the best of both worlds and immediately gives our final result.
\begin{theorem}
\label{t:final}
If $\Pol$ is a countable set of policies, then the ROI of \esc-\cape{} run for $N$ tasks with confidence parameter $\delta\in(0,1)$, constant accuracy levels $\e_j = N^{-1/3}$, and exploration cap $\Nex = \lce{ N^{2/3} }$ converges to the optimal $\roi(\pi_{\ks})$, with probability at least $1-\delta$, at a rate
\[
    R_N
=
    \wt{\cO} \lrb{  \frac{ 1 + \capf{K} \I\{\text{\esc{} halts returning $K$}\}}{N^{1/3}} }
\]
where the $\wt{\cO}$ notation hides only logarithmic terms, including a $\log(1/\delta)$ term.
\end{theorem}

\section{Conclusions and Open Problems}
\label{s:conclusions}
After formalizing the problem of ROI maximization in repeated decision making, we presented an algorithm (\esc{}-\cape{}) that is competitive against infinitely large policy sets (\Cref{t:final}).
For this algorithm, we prove a convergence rate of order $1/N^{1/3}$ with high probability.
To analyze it, we first proved a convergence result for its finite counterpart \cape{} (\Cref{t:early-stop}), which is of independent interest.
Notably, this finite analysis guarantees a significantly faster convergence of order $1/N$ on easier instances in which there is a positive gap in performance between the two best policies.

\section*{Acknowledgements}

An earlier version of this work was done during Tommaso Cesari's Ph.D. at the University of Milan.
Nicol\`o Cesa-Bianchi and Tommaso R. Cesari gratefully acknowledge partial support by Criteo AI Lab through a Faculty Research Award and by the MIUR PRIN grant Algorithms, Games, and Digital Markets (ALGADIMAR).
This work has also benefited from the AI Interdisciplinary Institute ANITI. ANITI is funded by the French ``Investing for the Future – PIA3'' program under the Grant agreement n. ANR-19-PI3A-0004.
Yishay Mansour was supported in part by a grant from the Israel Science Foundation (ISF).
Vianney Perchet was supported by a public grant as part of the Investissement d’avenir project, reference ANR-11-LABX-0056-LMH, LabEx LMH, in a joint call with Gaspard Monge Program for optimization, operations research and their interactions with data sciences. Vianney Perchet also acknowledges the support of the ANR under the grant ANR-19-CE23-0026.

\appendix

\section{Decision-Making Policies}
\label{s:policyFormal}

In this section, we give a formal functional definition of the decision-making policies introduced in \Cref{s:setting}.
During each task, the agent sequentially observes samples $x_i \in \spino$ representing realizations of stochastic observations of the current innovation value.
A map
$
    \durationrm\colon \spinon \to \N    
$ 
is a \textsl{\duration} (of a decision task) if for all $\bx \in \spinon$, its value $d = \durationrm(\bx)\in \N$ at $\bx$ depends only on the first $d$ components $x_1, x_2, \l, x_{d}$ of $\bx=(x_1, x_2, \l)$; mathematically speaking, if $\bX$ is a discrete stochastic process (i.e., a random sequence), then $\tau(\bX)$ is a stopping time with respect to the filtration generated by $\bX$.
This definition reflects the fact that the components $x_1, x_2,\l$ of the sequence $\bx = (x_1, x_2, \l)$ are generated sequentially, and the decision to stop testing an innovation depends only on what occurred so far. 
A concrete example of a \duration{} function is the one, mentioned in the introduction and formalized in~\eqref{e:pol-hoeff1}, that keeps drawing samples until the empirical average of the observed values $x_i$ surpasses/falls below a certain threshold, or a maximum number of samples have been drawn.

To conclude a task, the agent has to make a decision: either accepting or rejecting the current innovation. Formally, we say that a function
$
    \accept \colon \N \times \spinon \to \bool
$
is a \textsl{\decision} (to accept) if for all $d\in \N$ and $\bx \in \spinon$, its value $\accept(d,\bx)\in \bool$ at $(d,\bx)$ depends only on the first $d$ components $x_1,\l,x_d$ of $\bx = (x_1, x_2, \l)$.
Again, this definition reflects the fact that the decision $\accept(d,\bx)$ to either accept ($\accept(d,\bx) = 1$) or reject ($\accept(d,\bx) = 0$)  the current innovation after observing the first $d$ values $x_1, \l, x_d$ of $\bx=(x_1,x_2,\l)$  is oblivious to all future observations $x_{d+1}, x_{d+2}, \l$.
Following up on the concrete example above, the \decision{} function is accepting the current innovation if and only if the the empirical average of the observed values $x_i$ surpasses a certain threshold.\footnote{Note that, even for \decision{} functions that only look at the mean of the first $d$ values, our definition is more general than simple threshold functions of the form $\I \{ \text{mean} \ge \e_d \}$,
as it also includes all \decision{}s of the form $\I \{ \text{mean} \in A_d \}$, for all measurable $A_d \s \R$.
}

Since the only two choices that an agent  makes in a decision task are when to stop drawing new samples and whether or not to accept the current innovation,  the behavior of the agent during each task is fully characterized by the choice of a pair
$
    \pi 
= 
    (\durationrm, \accept)
$
that we call a (\textsl{decision-making}) \textsl{policy}, where $\durationrm$ is a \duration{} 
and $\accept$ is a \decision{}.

\section{Technical Lemmas for Theorem~\ref{t:early-stop}}
\label{s:techlem}

In this section, we give formal proofs of all results needed to prove Theorem~\ref{t:early-stop}.

\begin{lemma}
\label{lm:claim1}
Under the assumptions of Theorem~\ref{t:early-stop},
the event
\begin{equation}
    \label{e:claim1-lemma}
    \rmm_n(k) \le \rewrm(\pi_k) \le \rp_n(k)
\qquad\text{and}\qquad
    \sm_n(k) \le \cosrm(\pi_k) \le \spp_n(k)
\end{equation}
occurs simultaneously for all $n = 1,\l,\Nex$ and all $k=1,\l,\max(C_n)$ with probability at least $1-\delta$. 
\end{lemma}
\begin{proof}
Let, for all $n,k$,
\begin{equation}
    \label{e:bars-lemma}
    \e_n = \sqrt{\frac{\ln(4K\Nex/\delta)}{2n}},
\qquad
    \rewbar_n(k) = \rp_n(k) - 2\e_n,
\qquad
    \cosbar_n(k) = \spp_n(k) - (\capf{k}-1) \e_n
\end{equation}
Note that $\cosbar_n(k)$ is the empirical average of $n$ i.i.d.\ samples of $\cosrm(\pi_k)$ for all $n,k$ by definitions \eqref{e:bars-lemma}, \eqref{e:cpmdef}, \eqref{e:reward}, \eqref{e:rewCostShort}, and point \ref{i:samples} in the formal definition of our protocol (Section~\ref{s:setting}).
We show now that $\rewbar_n(k)$ is the empirical average of $n$ i.i.d.\ samples of $\rewrm(\pi_k)$ for all $n,k$; then claim \eqref{e:claim1-0} follows by Hoeffding's inequality.
Indeed, by the conditional independence of the samples and being $\accept(k, \bx)$ independent of the variables $(x_{k+1},x_{k+2},\l)$ by 
definition,
for all tasks $n$, all policies $k\in C_n$, and all $i> \capf{\max(C_n)}$ ($\ge \capf{k}$ by monotonicity of $k\mapsto\capf{k}$), 
\begin{align*}
    \E \lsb{ X_{n,i} \,  \accept \brb{ \durationrm_{k} (\bX_n), \bX_n} \,\Big|\, \mu_n }
& =
	\E \lsb{
	X_{n,i}
	\mid
	\mu_n
	}
	\E \Bsb{
		\accept \brb{ \durationrm_{k} (\bX_n), \bX_n}
	\,\Big|\,
	\mu_n	
	}
\\ & =
	\mu_n
	\, 
	\E \Bsb{
		\accept \brb{ \durationrm_{k} (\bX_n), \bX_n}
	\,\Big|\,
	\mu_n	
	}
\\ & =
	\E \Bsb{
		\mu_n
		\, 
		\accept \brb{ \durationrm_{k} (\bX_n), \bX_n}
		\,\Big|\,
		\mu_n	
	}
\end{align*}
Taking expectations with respect to $\mu_n$ on both sides of the above, and recalling definitions \eqref{e:bars-lemma}, \eqref{e:rpmdef}, \eqref{e:reward}, \eqref{e:rewCostShort}, \eqref{i:samples} proves the claim. 
Thus, Hoeffding's inequality implies, for all fixed $n,k$,
\begin{align*}
    \P \brb{ \rmm_n(k) \le \rewrm(\pi_k) \le \rp_n(k) }
&
    = \P \Brb{ \bab{ \rewbar_n(k) - \rewrm(\pi_k) } \le 2\e_n }
    \ge 1 - \frac{\delta}{2 K \Nex}
\\
    \P \brb{ \sm_n(k) \le \cosrm(\pi_k) \le \spp_n(k) }
&
    = \P \Brb{ \bab{ \cosbar_n(k) - \cosrm(\pi_k) } \le (\capf{K}-1)\e_n }
    \ge 1 - \frac{\delta}{2 K \Nex}
\end{align*}
Applying a union bound shows that event \eqref{e:claim1-0} occurs simultaneously for all $n \in \{1,\l,\Nex\}$ and $k\in\{1,\l,\max(C_n)\}$ with probability at least $1-\delta$.
\end{proof}

\begin{lemma}
\label{lm:claim2}
Under the assumptions of Theorem~\ref{t:early-stop}, if the event \eqref{e:claim1-lemma} occurs simultaneously for all $n = 1,\l,\Nex$ and all $k=1,\l,\max(C_n)$, and $\Delta>0$, (i.e., if there is a unique optimal policy), then all suboptimal policies are eliminated after at most $\Nex'$ tasks, where
\begin{equation}
    \label{e:claim2-lemma}
    \Nex'
\le 
    \frac{288 \, \capf{K}^2 \ln(4K\Nex /\delta)}{\Delta^2} + 1
\end{equation}
\end{lemma}
\begin{proof}
Note first that \eqref{e:claim1-lemma} implies, for all $n \ge 2 \capf{K}^2 \ln(4K\Nex/ \delta)$ (guaranteed by line~\ref{s:polel}) and all $k\in C_n$
\begin{align*}
    \frac{\rmm_n(k)}{\spp_n(k)}
& \le
    \frac{\rewrm(\pi_k)}{\cosrm(\pi_k)}
\le
    \frac{\rp_n(k)}{\sm_n(k)}
    \qquad
    \text{ if } \rp_n(k) \ge 0
\\
    \frac{\rmm_n(k)}{\sm_n(k)}
& \le
    \frac{\rewrm(\pi_k)}{\cosrm(\pi_k)}
\le
    \frac{\rp_n(k)}{\spp_n(k)}
    \qquad
    \text{ if } \rp_n(k) < 0
\end{align*}
In other words, the interval 
\[
    \lsb{ 
    \frac{\rmm_n(k)}{\spp_n(k)}\I\bcb{\rp_n(k) \ge 0}
    + \frac{\rmm_n(k)}{\sm_n(k)}\I\bcb{\rp_n(k) < 0}
, \ 
    \frac{\rp_n(k)}{\sm_n(k)}\I\bcb{\rp_n(k) \ge 0} 
    + \frac{\rp_n(k)}{\spp_n(k)}\I\bcb{\rp_n(k) < 0}
    }
\]
is a confidence interval for the value ${\rewrm(\pi_k)}/{\cosrm(\pi_k)}$ that measures the performance of $\pi_k$.
Let, for all $n,k$,
\begin{equation}
    \label{e:bars-lemma2}
    \e_n = \sqrt{\frac{\ln(4K\Nex/\delta)}{2n}},
\qquad
    \rewbar_n(k) = \rp_n(k) - 2\e_n,
\qquad
    \cosbar_n(k) = \spp_n(k) - (\capf{k}-1) \e_n
\end{equation}
If $\rp_n(k) \ge 0$, by the definitions in \eqref{e:bars-lemma2}, the length of this confidence interval is
\[
    \frac{\rewbar_n(k)+2\e_n}{\cosbar_n(k)-(\capf{k}-1)\e_n}
    -\frac{\rewbar_n(k)-2\e_n}{\cosbar_n(k)+(\capf{k}-1)\e_n}
=
    \frac{2 \e_n \brb{ 2\,\cosbar_n(k)+(\capf{k}-1) \, \rewbar_n(k) }}{\cosbar_n(k)^2 - (\capf{k}-1)^2 \, \e_n^2}
\le
    12 \, \capf{K} \e_n
\]
where for the numerator we used the fact that $\cosbar_n(k)$ (resp., $\rewbar_n(k)$) is an average of random variables all upper bounded by $\capf{k}$ (resp., $1$) and the denominator is lower bounded by $1/2$ because $\cosbar_n(k)^2 \ge 1$, $(\capf{k}^2-1)\, \e_n^2 \le 1/2$ by $n \ge 2 \capf{K}^2 \ln(4K\Nex/ \delta)$ (line~\ref{s:nexbig}), and $\capf{k}/\capf{K} \le 1$ (by monotonicity of $k \mapsto \capf{k}$).
Similarly, if $\rp_n(k) < 0$, the length of the confidence interval is
\[
    \frac{\rewbar_n(k)+2\e_n}{\cosbar_n(k)+(\capf{k}-1)\e_n}
    -\frac{\rewbar_n(k)-2\e_n}{\cosbar_n(k)-(\capf{k}-1)\e_n}
=
    \frac{2 \e_n \brb{ 2\,\cosbar_n(k)-(\capf{k}-1) \, \rewbar_n(k) }}{\cosbar_n(k)^2 - (\capf{k}-1)^2 \, \e_n^2}
\le
    12 \, \capf{K} \e_n
\]
where, in addition to the considerations above, we used $0 < -\rp_n(k) < -\rewbar_n(k) \le 1$.
Hence, as soon as the upper bound $12\,\capf{K}\e_n$ on the length of each of the confidence interval above falls below $\Delta/2$, all such intervals are guaranteed to be disjoint and by definition of $C_n$ (line~\ref{s:polel}), all suboptimal policies are guaranteed to have left $C_{n+1}$. In formulas, this happens at the latest during task $n$, where $n \ge 2 \capf{K}^2 \ln(4K\Nex/ \delta)$ satisfies
\[
    12 \, \capf{K} \e_n < \frac{\Delta}{2}
\iff
    n > 288 \, ( \capf{K} /{\Delta} )^2 \ln(4K\Nex /\delta)
\]
This proves the result. 
\end{proof}

\begin{lemma}
\label{lm:claim3}
Under the assumptions of Theorem~\ref{t:early-stop}, if the event \eqref{e:claim1-lemma} occurs simultaneously for all $n = 1,\l,\Nex$ and all $k=1,\l,\max(C_n)$,
and 
the test at line~\ref{s:testPolElim} is true for some $\Nex' \le \Nex$, 
then
\begin{equation}
    \label{e:claim3-lemma}
    R_N
\le
    \min \lrb
    {\frac{(2 \capf{K}  + 1) \Nex }{N}
    ,\
    \frac
    {( 2\capf{K}  + 1 ) \brb{ 288 \, ( {\capf{K}}/{\Delta} )^2 \ln(4K\Nex /\delta) + 1 }}
    {N}
    }
\end{equation}
\end{lemma}

\begin{proof}
Note that if the test at line~\ref{s:testPolElim} is true, than by \eqref{e:claim1-lemma} there exists a unique optimal policy, i.e., we have $\Delta>0$. We can therefore apply Lemma~\ref{lm:claim2}, obtaining a deterministic upper bound $\Nex''$ on the number $\Nex'$ of tasks needed to identify the optimal policy, where
\[
    \Nex''
=
    \min\lrb{\Nex, \ \frac{128 \, \capf{K}^2 \ln(4K\Nex /\delta)}{\Delta^2} + 1 }
\]
The total expected reward of Algorithm~\ref{algo:cape} divided by its total expected cost is lower bounded by
\[
    \xi
=
    \frac{ \E \lsb{ -\Nex' + \sum_{n=\Nex'+1}^N \rewardrm(\pi_{\ks},\mu_n)} } { \E \lsb{ 2\sum_{m=1}^{\Nex'} \capf{\max(C_m)} + \sum_{n=\Nex'+1}^N \costrm(\pi_{\ks},\mu_n) } }
\]
If $\xi < 0$, we can further lower bound it by
\[
    \frac{(N-\Nex'') \, \rewrm(\pi_{\ks}) -\Nex''}{(N-\Nex'') \, \cosrm(\pi_{\ks}) + 2 \Nex''}
\ge
    \frac{\rewrm(\pi_{\ks})}{\cosrm(\pi_{\ks})} - \frac{3 \Nex''}{N}
\]
where the inequality follows by $(a-b)/(c+d) 
\ge a/c - (d+b)/(c+d)$ for all $a,b,c,d\in \R$ with $0\neq c > -d$ and $a/c \le 1$, and then using $c+d\ge N$ which holds because $\cosrm(\pi_{\ks}) \ge 1$.
Similarly, if $\xi \ge 0$, we can further lower bound it by
\[
    \frac{(N-\Nex'') \, \rewrm(\pi_{\ks}) -\Nex''}{(N-\Nex'') \, \cosrm(\pi_{\ks}) + 2 \capf{K} \Nex''}
\ge
    \frac{\rewrm(\pi_{\ks})}{\cosrm(\pi_{\ks})} - \frac{(2\capf{K}+1) \Nex''}{N}
\]
Thus, the result follows by $\capf{K}\ge 1$ and the definition of $\Nex''$.
\end{proof}

\begin{lemma}
\label{lm:claim4}
Under the assumptions of Theorem~\ref{t:early-stop}, if the event \eqref{e:claim1-lemma} occurs simultaneously for all $n = 1,\l,\Nex$ and all $k=1,\l,\max(C_n)$,
and 
the test at line~\ref{s:testPolElim} is
false for all tasks $n \le \Nex$ (i.e., if line~\ref{s:exploit} is executed with $C_{\Nex+1}$ containing two or more policies),
then
\[
    R_T
\le
    (\capf{K}+1)\sqrt{\frac{8 \ln(4K\Nex/\delta)}{\Nex}}
+ 
     \frac{ (2 \capf{K} + 1)\Nex }{ N }
\]
\end{lemma}

\begin{proof}
Note first that by \eqref{e:claim1-lemma} and the definition of $C_n$ (line~\ref{s:polel}), all optimal policies belong to $C_{\Nex+1}$. 
Let, for all $n,k$,
\begin{equation}
    \label{e:bars-lemma4}
    \e_n = \sqrt{\frac{\ln(4K\Nex/\delta)}{2n}},
\qquad
    \rewbar_n(k) = \rp_n(k) - 2\e_n,
\qquad
    \cosbar_n(k) = \spp_n(k) - (\capf{k}-1) \e_n
\end{equation}
By \eqref{e:claim1-lemma} and the definitions of $\kts$, $\rpm_n(k)$, and $\e_n$ (line~\ref{s:exploit}, \eqref{e:rpmdef}, \eqref{e:rpmdef}, and \eqref{e:bars-lemma4} respectively), for all optimal policies $\pi_{\ks}$, if $\rp_{\Nex}(\ks) \ge 0$, then also $\rp_{\Nex}(\kts) \ge 0$\footnote{Indeed, $\kts \in \argmax_{k \in C_{\Nex+1}}  \brb{ {\rp_{\Nex}(k)}/{\sm_{\Nex}(k)} }$ in this case, and $\rp_{\Nex}(\kts)\ge 0$ follows by the two inequalities ${\rp_{\Nex}(\kts)}/{\sm_{\Nex}(\kts)} \ge {\rp_{\Nex}(\ks)}/{\sm_{\Nex}(\ks)} \ge 0$.}  and
\begin{multline*}
	\frac{\rewrm(\pi_{\ks})}{\cosrm(\pi_{\ks})}
\le
	\frac{\rp_{\Nex}(\ks)}{\sm_{\Nex}(\ks)}
\le
	\frac{\rp_{\Nex}(\kts)}{\sm_{\Nex}(\kts)}
\le
	\frac{\rewrm(\pi_{\kts}) + 4 \e_n}{\cosrm(\pi_{\kts}) - 2 (\capf{\kts}-1) \e_n}
\\
\le
	\frac{\rewrm(\pi_{\kts})}{\cosrm(\pi_{\kts})}	
	+\frac{2(\capf{\kts}+1)\e_n}{\cosrm(\pi_{\kts}) - 2 (\capf{\kts}-1) \e_n}
\end{multline*} 
where all the denominators are positive because $\Nex \ge 8 (\capf{K}-1)^2 \ln(4 K \Nex / \delta)$ and the last inequality follows by $(a+b)/(c-d) \le a/c + (d+b)/(c-d)$ for all $a\le 1$, $b\in\R$, $c\ge 1$, and $d<c$; 
next, if $\rp_{\Nex}(\ks)<0$ but $\rp_{\Nex}(\kts)\ge 0$ the exact same chain of inequalities hold; finally, if both $\rp_{\Nex}(\ks)<0$ and $\rp_{\Nex}(\kts)< 0$, then $\rp_{\Nex}(k) < 0$ for all $k\in C_{\Nex+1}$\footnote{Otherwise $\kts$ would belong to the set $\argmax_{k \in C_{\Nex+1}}  \brb{ {\rp_{\Nex}(k)}/{\sm_{\Nex}(k)} }$ which in turn would be included in the set $\bcb{k \in C_{\Nex+1} : \rp_{\Nex}(k) \ge 0}$ and this would contradict the fact that $\rp_{\Nex}(\kts) < 0$.}, hence, by definition of $\kts$ and the same arguments used above
\begin{multline*}
	\frac{\rewrm(\pi_{\ks})}{\cosrm(\pi_{\ks})}
\le
	\frac{\rp_{\Nex}(\ks)}{\spp_{\Nex}(\ks)}
\le
	\frac{\rp_{\Nex}(\kts)}{\spp _{\Nex}(\kts)}
\le
	\frac{\rewrm(\pi_{\kts}) + 4 \e_n}{\cosrm(\pi_{\kts}) + 2 (\capf{\kts}-1) \e_n}
\\
\le
	\frac{\rewrm(\pi_{\kts})}{\cosrm(\pi_{\kts})}	
	+\frac{2(\capf{\kts}+1)\e_n}{\cosrm(\pi_{\kts}) + 2 (\capf{\kts}-1) \e_n}
\le
	\frac{\rewrm(\pi_{\kts})}{\cosrm(\pi_{\kts})}	
	+\frac{2(\capf{\kts}+1)\e_n}{\cosrm(\pi_{\kts}) - 2 (\capf{\kts}-1) \e_n}
\end{multline*} 
That is, for all optimal policies $\pi_{\ks}$, the policy $\pi_{\kts}$ run at line~\ref{s:exploit} satisfies
\begin{multline*}
    \rewrm(\pi_{\kts})
\ge
	\cosrm(\pi_{\kts}) \lrb{ \frac{\rewrm(\pi_{\ks})}{\cosrm(\pi_{\ks})} - \frac{2(\capf{\kts}+1)\e_n}{\cosrm(\pi_{\kts}) - 2 (\capf{\kts}-1) \e_n} }
\\ \ge
	\cosrm(\pi_{\kts}) \lrb{ \frac{\rewrm(\pi_{\ks})}{\cosrm(\pi_{\ks})} - 4(\capf{K} +1)\e_n }
\end{multline*}
where in the last inequality we lower bounded the denominator by $1/2$ using $\cosrm(\pi_{\kts})\ge 1$ and $\e_n \le \e_{\Nex} \le 1/2$ which follows by $n \ge \Nex \ge 8 \capf{K}^2 \ln(4 K \Nex / \delta)$ and the monotonicity of $k\mapsto \capf{k}$.
Therefore, for all optimal policies $\pi_{\ks}$, the total expected reward of Algorithm~\ref{algo:cape} divided by its total expected cost (i.e., the negative addend in \eqref{e:regret}) is at least
\begin{multline*}
    \frac{ \E \bsb{ -\Nex + (N-\Nex) \, \rewrm(\pi_{\kts}) } } { \E \bsb{ 2\sum_{n=1}^{\Nex}\capf{\max(C_n)} + (N-\Nex) \, \cosrm(\pi_{\kts}) } }
\\
    \begin{aligned}
    & \ge 
        \frac{-\Nex}{  2\sum_{n=1}^{\Nex} \E \bsb{ \capf{\max(C_n)} } + (N-\Nex) \, \E \bsb{ \cosrm(\pi_{\kts}) } }
    \\
    & +
        \frac{(N-\Nex) \, \E \bsb{ \cosrm(\pi_{\kts}) } } { 2\sum_{n=1}^{\Nex} \E \bsb{ \capf{\max(C_n)} } + (N-\Nex) \, \E \bsb{ \cosrm(\pi_{\kts}) } }
        \lrb{ \frac{\rewrm(\pi_{\ks})}{\cosrm(\pi_{\ks})} - 4(\capf{K} +1)\e_n }
    \\ 
    & \ge
        \frac{\rewrm(\pi_{\ks})}{\cosrm(\pi_{\ks})} - 4(\capf{K} +1)\e_n
        - \frac{\Nex + 2\sum_{n=1}^{\Nex} \E \bsb{ \capf{\max(C_n)} } }{ 2\sum_{n=1}^{\Nex} \E \bsb{ \capf{\max(C_n)} } + (N-\Nex) \, \E \bsb{ \cosrm(\pi_{\kts}) } }
    \\
    & \ge
        \frac{\rewrm(\pi_{\ks})}{\cosrm(\pi_{\ks})} - 4(\capf{K} +1)\e_n
        - \frac{ (2 \capf{K} + 1)\Nex }{ N }
    \end{aligned}
\end{multline*}
where we used $\frac{a}{b+a}(x-y) \ge x-y - \frac{b}{b+a}$ for all $a,b,y>0$ and all $x\le 1$ to lower bound the third line, then the monotonicity of $k\mapsto \capf{k}$ and $2 \E \bsb{ \capf{\max(C_n)} } \ge \E \bsb{ \cosrm(\pi_{\kts}) } \ge 1$ for the last inequality. 
Rearranging the terms of the first and last hand side in the previous display, using the monotonicity of $k\mapsto \capf{k}$, and plugging in the value of $\e_n$, gives
\[
    R_T
\le
    4(\capf{K} +1)\e_n 
+ 
    \frac{ (2 \capf{K} + 1)\Nex }{ N }
=
    (\capf{K}+1)\sqrt{\frac{8 \ln(4K\Nex/\delta)}{\Nex}}
+ 
     \frac{ (2 \capf{K} + 1)\Nex }{ N }
\]
\end{proof}

\section{A Technical Lemma for Theorem~\ref{t:final}}
\label{s:techlem-countable}

In this section, we give a formal proof for a result needed to prove \Cref{t:final}.

\lemmaKcountable*
\begin{proof}
Note fist that $\rmm_{2^j}+2\e_j$ (line~\ref{s:rj}) is an empirical average of $m_j$ i.i.d.\ unbiased estimators of $\rewrm (\pi_{2^j})$. Indeed, being $\accept(k, \bx)$ independent of the variables $(x_{k+1},x_{k+2},\l)$ 
by definition of duration 
and the conditional independence of the samples (recall the properties of samples in step~\ref{i:samples} of our online protocol, Section~\ref{s:setting}), for all tasks $n$ performed at line~\ref{s:rj} during iteration $j$ and all $i > \capf{2^j}$,
\begin{multline*}
    \E \lsb{ X_{n,i} \,  \accept \brb{ \durationrm_{2^j} (\bX_n), \bX_n} \,\Big|\, \mu_n }
=
	\E \lsb{
	X_{n,i}
	\mid
	\mu_n
	}
	\E \Bsb{
		\accept \brb{ \durationrm_{2^j} (\bX_n), \bX_n}
	\,\Big|\,
	\mu_n	
	}
\\
=
	\mu_n
	\, 
	\E \Bsb{
		\accept \brb{ \durationrm_{2^j} (\bX_n), \bX_n}
	\,\Big|\,
	\mu_n	
	}
=
	\E \Bsb{
		\mu_n
		\, 
		\accept \brb{ \durationrm_{2^j} (\bX_n), \bX_n}
		\,\Big|\,
		\mu_n	
	}
\end{multline*}
Taking expectations to both sides proves the claim. 
Thus, Hoeffding's inequality implies 
\[
    \P \Brb{
	\rmm_{2^j}
	>
	\rewrm ( \pi_{2^j} )
	} 
= 
	\P \Brb{
	\brb{	
	\rmm_{2^j}
	+2\e_j }
	- \rewrm ( \pi_{2^j})
	>
	2\e_j
	}
\le 
	\frac{\delta}{j(j+1)}
\]
for all $j\le j_0$. 
Similarly, for all $l>j_0$, 
$
	\P \brb{ \sbar_{2^l} - \cosrm (\pi_{2^l})  > \capf{2^l} \,  \e_l } \le \frac{\delta}{l(l+1)}
$. 
Hence, the event
\begin{equation}
\label{e:good-event-phase-1}
	\bcb{ \rmm_{2^j} \le\rewrm ( \pi_{2^j}) }
	\; \wedge \;
	\bcb{ \sbar_{2^l} \le \cosrm (\pi_{2^l})) + \capf{2^l} \,  \e_l}
	\qquad
	\forall j \le j_0, \forall l > j_0
\end{equation}
occurs with probability at least 
\[
	1-\sum_{j=1}^{j_0} \frac{\delta}{j(j+1)}- \sum_{l=j_0+1}^{j_1}\frac{\delta}{l(l+1)}
\ge 
	1 - \delta \sum_{j\in\N} \frac{1}{j(j+1)}
= 
	1 - \delta
\] 
Note now that for each policy $\pi_k$ with $\rewrm(\pi_k)\ge 0$ and each optimal policy $\pi_{\ks}$, 
\begin{equation}
    \label{e:reduClosArg}
    \frac{\rewrm(\pi_k)}{\capf{k}}
\le
	\frac{\rewrm(\pi_k)}{\cosrm(\pi_k)}
\le
	\frac{\rewrm(\pi_{\ks})}{\cosrm(\pi_{\ks})}
\le
	\frac{1}{\cosrm(\pi_{\ks})}
\end{equation}
Hence, all optimal policies $\pi_{\ks}$ satisfy $\cosrm(\pi_{\ks}) \le \capf{k}/\rewrm(\pi_k)$ for all policies $\pi_k$ such that $\rewrm(\pi_k) > 0$. 
Being \duration{}s sorted by index, for all $k\le h$
\begin{equation}
\label{e:earl-stop-monot}
	\cosrm(\pi_k)
=
	\E \bsb{ \costrm(\pi_k,\mu_n) }
\le
	\E \bsb{ \costrm(\pi_h,\mu_n) }
=
	\cosrm(\pi_h)
\end{equation}
Thus, with probability at least $1 - \delta$, for all $k>K$
\[
	\cosrm(\pi_k)
\overset{\eqref{e:earl-stop-monot}}{\ge}
	\cosrm(\pi_K)
\overset{\eqref{e:good-event-phase-1}}{\ge}
	\sbar_{K} - \capf{K} \,  \e_{\log_2 K}
\overset{\text{line~}\ref{st:test-p2}}{>}
	\frac{\capf{k_0}}{\rmm_{k_0}}
\ge
	\frac{\capf{k_0}}{\rewrm(k_0)}
\] 
where $\rewrm(k_0) \ge \rmm_{k_0} > 0$ by \eqref{e:good-event-phase-1} and line~\eqref{st:test-p1}; i.e., $\pi_k$ do not satisfy \eqref{e:reduClosArg}. Therefore, with probability at least $1-\delta$, all optimal policies $\pi_{\ks}$ satisfy $\ks \le K$.
\end{proof}

\section{Choice of Performance Measure}
\label{s:model-choice}
In this section, we discuss our choice of measuring the performance of  policies $\pi$ with
\[
    \frac{\sum_{n=1}^N \E \bsb{ \rewardrm(\pi,\mu_n) } }
        {\sum_{m=1}^N \E \bsb{ \costrm(\pi,\mu_m) } }
=
    \frac{\rewrm(\pi)}{\cosrm(\pi)}
\]
We compare several different benchmarks and investigate the differences if the agent had a budget of samples and a variable number of tasks, rather than the other way around. We will show that all ``natural'' choices  essentially go in the same direction, except for one (perhaps the most natural) which turns out to be the worst.

At a high level, an agent constrained by a budget would like to maximize its ROI. This can be done in several different ways. 
If the constraint is on the number $N$ of tasks, then the agent could aim at  maximizing (over $\pi = (\durationrm, \accept) \in \Pol$) the objective $g_1(\pi,N)$ defined by
\[
	g_1(\pi, N)
=
    \E \lsb{ \frac{\sum_{n=1}^N \rewardrm(\pi,\mu_n)}{\sum_{m=1}^N \costrm(\pi,\mu_m)} }
\] 
This is equivalent to the maximization of the ratio
\[
    \frac{\rewrm(\pi)}{\cosrm(\pi)}
=
    \frac{ \E \bsb{ \rewardrm(\pi,\mu_n) } }{ \E \bsb{ \costrm(\pi,\mu_n) } }
\]
in the sense that, multiplying both the numerator and the denominator in $g_1(\pi,N)$ by $1/N$ and applying Hoeffding's inequality, we get $g_1(\pi,N) = \Theta\brb{ \rewardrm(\pi) / \costrm(\pi) }$. Furthermore, by the law of large numbers and Lebesgue's dominated convergence theorem, $g_1(\pi,N) \to \rewardrm(\pi) / \costrm(\pi)$ when $N\to \iop$ for any $\pi\in \Pol$. 

Assume now that the constraint is on the total number of samples instead. We say that the agent has a \textsl{budget of samples} $T$ if as soon as the total number of samples reaches $T$ during task $N$ (which is now a random variable), the agent has to interrupt the run of the current policy, reject the current \mutation{} $\mu_N$, and end the process. Formally, the random variable $N$ that counts the total number of tasks performed by repeatedly running a policy $\pi = (\durationrm, \accept)$ is defined by
\[
	N 
= 
	\min \lcb{ m\in \N \,\bigg|\, \sum_{n=1}^m \durationrm(\bX_n) \ge T }
\]
In this case, the agent could aim at maximizing
the objective
\[
	g_2(\pi,T)
=
    \E \lsb{ \frac{\sum_{n=1}^{N-1}\rewardrm(\pi, \mu_n)}{T} }
\] 
where the sum is $0$ if $N=1$ and it stops at $N-1$ because the the last task is interrupted and no reward is gained. 
As before, assume that $\durationrm \le D$, for some $D \in \N$.
Note first that by the independence of $\mu_n$ and $\bX_n$ from past tasks, for all deterministic functions $f$ and all $n\in\N$, the two random variables $f(\mu_n,\bX_n)$ and 
$
	\I\{N \ge n\} 
$ 
are independent, because 
$
	\I\{N \ge n\} = \I \bcb{ \sum_{i=1}^{n-1} \durationrm(\bX_i) < T }
$ 
depends only on the random variables $\durationrm(\bX_1), \l,\durationrm(\bX_{n-1})$. 
Hence
\begin{align*}
	\E \Bsb{ 
	\rewardrm(\pi,\mu_n)
	\,  \I\{N \ge n\} }
& =
    \rewrm(\pi) \,  \P(N \ge n)\\
	\E \bsb{ 
	\costrm(\pi,\mu_n)
	\,  \I\{N \ge n\} }
& =
    \cosrm(\pi) \,  \P(N \ge n)
\end{align*}
Moreover, note that during each task at least one sample is drawn, hence $N\le T$ and
\begin{align*}
	\sum_{n=1}^\iop \E \Bsb{ \bab{ 
	\rewardrm(\pi,\mu_n)
	} \,  \I\{N \ge n\} }
& \le
	\sum_{n=1}^T \E \Bsb{ \big| \rewardrm(\pi,\mu_n) \big| }
\le 
	T
<
	\iop\\
	\sum_{n=1}^{\iop} \E \bsb{
	\costrm(\pi,\mu_n)
	\,  \I \{N\ge n\} }
& \le
	\sum_{n=1}^{T} \E \bsb{ \costrm(\pi,\mu_n) }
=
	 T \, \cosrm(\pi)
\le
    T D
<
	\iop
\end{align*}
We can therefore apply Wald's identity \citep{wald1944cumulative} to deduce
\[
    \E \lsb{ \sum_{n=1}^N \rewardrm(\pi,\mu_n) }
=
	\E[N]\, \rewrm(\pi)
\qquad \text{and} \qquad
	\E \lsb{ \sum_{n=1}^N \cosrm(\pi,\mu_n) }
=
	\E[N]\, \cosrm(\pi)
\]
which, together with 
\[
	\E \lsb{ \sum_{n=1}^N 
	\costrm(\pi,\mu_n)
	}
\ge
    T
\ge
	\E \lsb{ 
	\sum_{n=1}^N 
	\costrm(\pi,\mu_n)
	} - 
	D
\] 
and 
\[
    \E \lsb{ \sum_{n=1}^N \rewardrm(\pi,\mu_n) }
    -1
\le
    \E \lsb{ \sum_{n=1}^{N-1} \rewardrm(\pi,\mu_n) }
\le
    \E \lsb{ \sum_{n=1}^N \rewardrm(\pi,\mu_n) }
    +1
\]
yields
\[
	\frac{\E[N] \,  \rewrm(\pi) - 1}{\E[N] \,  \cosrm(\pi)}
\le
	g_2(\pi,T)
\le
	\frac{\E[N] \,  \rewrm(\pi) + 1}{\E[N] \,  \cosrm(\pi) 
	- 
	D
	}
\]
if the denominator on the right-hand side is positive, which happens as soon as $T > D^2$ by
$ 
    N D
\ge
    \sum_{n=1}^N \durationrm(\bX_n) 
\ge 
    T
$
and $\cosrm(\pi) \ge 1$.
I.e., $g_2(\pi, T) = \Theta\brb{ \rewrm(\pi)/\cosrm(\pi) }$ and noting that $\E[N] \ge T/D \to \iop$ if $T\to \iop$, we have once more that $g_2(\pi, T) \to \rewrm(\pi)/\cosrm(\pi)$ when $T\to \iop$ for any $\pi\in \Pol$.

This proves that having a budget of tasks, samples, or using any of the three natural objectives introduced so far is essentially the same. 

Before concluding the section,
we go back to the original setting and discuss a very natural definition of objective which should be avoided because, albeit easier to maximize, it is not well-suited for this problem.
Consider as objective the average payoff of accepted \mutation{}s per amount of time used to make the decision, i.e.,
\[
	g_3(\pi)
=
	\E \lsb{ \frac{\rewardrm(\pi,\mu_n)}{\costrm(\pi,\mu_n)} }
\] 
We give some intuition on the differences between the ratio of expectations and the expectation of the ratio $g_3$ using the concrete example \eqref{e:pol-hoeff1} and we make a case for the former being better than the latter. 

More precisely, if $N$ decision tasks have to be performed by the agent, consider the natural policy class $\{\pi_k\}_{k\in\{1,\l,K\}} = \bcb{ (\durationrm_k, \accept) }_{k \in \{1,\l,K\}}$ given by 
\[
    \durationrm_k(\bx)
 =
    \min \lrb{k,\ \inf \lcb{ n \in \N : \lab{ \xbar_n } \ge c \sqrt{\frac{\ln \frac{KN}{\delta} }{n}}  } }
,\quad 
    \accept(n, \bx)
 =
    \I \lcb{ \xbar_n \ge c \sqrt{\frac{\ln \frac{KN}{\delta} }{n}}  }
\]
for some $c>0$ and $\delta\in(0,1)$, where $\xbar_n = (\nicefrac{1}{n})\sum_{i= 1}^n x_i$ is the average of the first $n$ elements
of the sequence $\bx= (x_1, x_2, \l)$.

If $K\gg 1$, there are numerous policies in the class with a large cap. For concreteness, consider the last one $(\durationrm_K, \accept)$
and let $k = \bce{ c^2 \ln(KN/\delta) }$.
If $\mu_n$ is uniformly distributed on $\{-1,0,1\}$, then
\[
	\Brb{ \durationrm_K(\bX_0), \accept\brb{ \durationrm_K(\bX_0), \bX_0 } }
=
	\begin{cases}
		(k, 1)	& \text{if } \mu_1 = 1\\
		(k, 0)	& \text{if } \mu_1 = -1\\
		(K, 0)	& \text{if } \mu_1 = 0
	\end{cases}
\]
i.e., the agent understands quickly (drawing only $k$ samples) that $\mu_n = \pm 1$, accepting it or rejecting it accordingly, but takes exponentially longer ($K \gg k$ samples) to figure out that the \mutation{} is nonpositive when $\mu_n = 0$. The fact that for a constant fraction of tasks ($1/3$ of the total) $\pi$ invests a long time ($K$ samples) to earn no reward makes it a very poor choice of policy. This is not reflected in the value of $g_3(\pi_K)$ but it is so in $\rewrm(\pi_K)/\cosrm(\pi_K)$. 
Indeed, in this instance
\[
	\E \lsb{ \frac{\rewardrm(\pi_K,\mu_n)}{\costrm(\pi_K,\mu_n)} }
=
	\Theta \lrb{ \frac{1}{k} }
\qquad \gg \qquad
	\Theta \lrb{ \frac{1}{K} }
=
	\frac{\rewrm(\pi_K)}{\cosrm(\pi_K)}
\]
This is due to the fact that the expectation of the ratio ``ignores'' outcomes with null (or very small) rewards, even if a large number of samples is needed to learn them. On the other hand, the ratio of expectations weighs the total number of requested samples and it is highly influenced by it, a property we are interested to capture within our model.

\section{An Impossibility Result}
\label{s:imposs}

We conclude the paper by showing that, in general, given $\mu_n$ it is impossible to define an unbiased estimator of the reward of all policies using only the samples drawn by the policies themselves, unless $\mu_n$ is known beforehand.

Take a policy $\pi_1 = (1, \accept)$ that draws exactly one sample. 
Note that such a policy is included in all sets of policies $\Pol$
so this is by no means a pathological example.
As before, assume for the sake of simplicity that samples take values in $\spin$ and consider any \decision{} function $\accept$ such that $\accept(1, \bx) = (1+x_1)/2$ for all $\bx = (x_1, x_2, \l )$.
In words, the policy $\pi_1$ looks at one single sample $x_1 \in \spin$ and accepts if and only if $x_1 = 1$.
As discussed earlier (Section~\ref{s:related}, Repeated A/B testing, and Section~\ref{s:model-choice}, where $\mu$ is concentrated around $[-1,0]\cup\{1\}$), there are settings in which this policy is optimal, so this choice of \decision{} function cannot be dismissed as a mathematical pathology.

The following lemma shows that in the simple, yet meaningful case of the policy $\pi_1$ described above, it is impossible to define an unbiased estimator of its expected reward given $\mu_n$
\[
    \E\bsb{ \mu_n \, \accept(1,\bX_{n}) \mid \mu_n }
=
    \mu_n \, \E\lsb{ \frac{ 1 + X_{n,1} }{2} \mid \mu_n}
=
    \frac{ \mu_n + \mu_n^2 }{2}
\]
using only $X_{n,1}$, unless $\mu_n$ is known beforehand.

\begin{lemma}
\label{lm:imposs}
Let $\tilde X$ be a $\spin$-valued random variable with $\E[ \tilde X ] = \tilde \mu$, for some real number $\tilde \mu$. 
If there exists an unbiased estimator $f(\tilde X)$ of $\brb{ \tilde \mu + \tilde \mu ^2 }/2$, for some $f\colon \spin \to \R$, then $f$ satisfies
\[
\begin{cases}
f(-1) = 0
    & \text{ if } \tilde \mu = -1 
\\[1ex]
f(1) = \dt{ \tilde \mu - f(-1) \frac{1-\tilde \mu}{1+\tilde\mu} }
    & \text{ if } \tilde \mu \neq -1
\end{cases}
\]
i.e., to define any such $f$ (thus, any unbiased estimator of $\brb{ \tilde \mu + \tilde \mu ^2 }/2$) it is necessary to know $\tilde \mu$.
\end{lemma}
\begin{proof}
From
$
    \E[\tilde X] 
= 
    1 \cdot \P(\tilde X=1) + (- 1)\cdot \P(\tilde X=-1) 
=
    - 1 + 2\P(\tilde X=1)
$
and our assumption $\E[ \tilde X ] = \tilde \mu$, we obtain $\P(\tilde X=1) = ( 1 + \tilde \mu )/2$.

Let $f\colon \spin \to \R$ be any function satisfying
$
    \E \bsb{ f(\tilde X) } = \brb{ \tilde \mu + \tilde \mu ^2 }/2
$.
Then, from the law of the unconscious statistician
\[
    \E \bsb{ f(\tilde X) }
=
    f(1) \P(\tilde X=1) + f(-1) \P(\tilde X=-1)
=
    f(1) \frac{1 + \tilde\mu}{2} + f(-1) \frac{1 - \tilde\mu}{2}
\]
and our assumption
$
    \E \bsb{ f(\tilde X) } = \brb{ \tilde \mu + \tilde \mu ^2 }/2
$, we obtain
\[
    f(1) (1 + \tilde\mu) + f(-1) (1 - \tilde\mu)
=
    \tilde \mu + \tilde \mu ^2
\]
Thus, if $\tilde \mu = -1$, we have $f(-1) = 0$. Otherwise, solving for $f(1)$ gives the result.
\end{proof}

} 

\bibliographystyle{ACM-Reference-Format}
\bibliography{bib1}

\end{document}